\documentclass[journal]{IEEEtran}

\usepackage{amsmath} % assumes amsmath package installed
\usepackage{amssymb}  % assumes amsmath package installed
\usepackage{graphicx}
\usepackage{epstopdf}
\usepackage{amsmath}
\usepackage{mathtools}
\usepackage{dsfont}
\usepackage{tikz}
\usepackage{siunitx}
\usepackage{hyperref}

\usepackage{caption}
\usepackage{subcaption}
\usepackage{floatrow}
\usepackage{balance}    % have even columns on last page
\usepackage{amsthm}
\usepackage[ruled,vlined,linesnumbered]{algorithm2e}
\usepackage{xcolor}
\usepackage{tcolorbox}
\usepackage{algpseudocode}% http://ctan.org/pkg/algorithmicx
\usepackage{dsfont}

\usepackage[noadjust]{cite} % multiple citations are combined
\usepackage{amsmath,stmaryrd,graphicx}
\usepackage{mathtools}

\newcommand{\norm}[1]{\left\lVert#1\right\rVert}
\newcommand{\R}{\mathbb{R}}
\renewcommand{\S}{\mathcal{S}}
\newcommand{\classK}{\mathcal{K}}

\newtheorem{theorem}{Theorem}

\newtheorem{proposition}{Proposition}

\theoremstyle{definition}
\newtheorem{definition}{Definition}
\theoremstyle{remark}
\newtheorem{remark}{Remark}
\theoremstyle{definition}

\theoremstyle{definition}
\newtheorem{example}{Example}

\newcommand{\newsec}[1]{
\vspace{0.2cm} 
\noindent \textbf{#1}
}

\definecolor{ugoColor}{rgb}{0.6,0.8,0.0}
\definecolor{ligthGray}{rgb}{0.95,0.95,0.95}

\makeatletter
\newcommand{\fixed@sra}{$\vrule height 2\fontdimen22\textfont2 width 0pt\shortrightarrow$}
\newcommand{\shortarrow}[1]{%
  \mathrel{\text{\rotatebox[origin=c]{\numexpr#1*45}{\fixed@sra}}}
}
\makeatother

\begin{document}

% \title{Multi-rate control: from POMDP planning \\ to low level actuation}

% \title{A Comparison of Control Barrier Functions and Artificial Potential Fields for Obstacle Avoidance}

\title{\LARGE \bf
Comparative Analysis of Control Barrier Functions and \newline Artificial Potential Fields for Obstacle Avoidance
}

\author{Andrew Singletary$^1$, Karl Klingebiel$^2$, Joseph Bourne$^2$, Andrew Browning$^2$, Phil Tokumaru$^2$, and Aaron Ames$^1$
\thanks{$^1$Mechanical and Civil Engineering, Caltech, Pasadena, CA, USA, 91125.
        {\tt\small \{ames,asinglet\}@caltech.edu}}
\thanks{$^2$Aerovironment, Simi Valley, CA, USA, 93065. \newline
        {\tt\small \{klingebiel,bourne,browning,tokumaru\}@avinc.com}}
        }

\maketitle

\begin{abstract}
Artificial potential fields (APFs) and their variants have been a staple for collision avoidance of mobile robots and manipulators for almost 40 years. Its model-independent nature, ease of implementation, and real-time performance have played a large role in its continued success over the years. Control barrier functions (CBFs), on the other hand, are a more recent development, commonly used to guarantee safety for nonlinear systems in real-time in the form of a filter on a nominal controller. 
In this paper, we address the connections between APFs and CBFs.  At a theoretic level, we prove that APFs are a special case of CBFs: given a APF one obtains a CBFs, while the converse is not true.  Additionally, we prove that CBFs obtained from APFs have additional beneficial properties and can be applied to nonlinear systems.
Practically, we compare the performance of APFs and CBFs in the context of obstacle avoidance on simple illustrative examples and for a quadrotor, both in simulation and on hardware using onboard sensing.  
These comparisons demonstrate that CBFs outperform APFs. 
% Practically, we compare the performance of APFs and CBFs in the context of obstacle avoidance on simple illustrative examples and for quadrotors.  In the latter case, we show in simulation and on hardware that CBFs outperform APFs in the context of providing smooth behaviors that are minimally invasive while guaranteeing obstacle avoidance. 
\end{abstract}

\IEEEpeerreviewmaketitle

\section{Introduction}

As mobile robots become increasingly popular in society, many hobbyists and companies are tasked with the challenge of keeping their robots from colliding with people, objects, and other robots. A real-time obstacle avoidance framework is necessary for ensuring that the operator is able to safely utilize the product. The goal of such a framework is to provide safety, but minimally alters the behavior of the robot when it is far away from any potential collisions. While many options exist for this task, few are as simple, easy to implement, and well-established as potential fields.

Artificial  Potential  Fields (APFs)  have  been  utilized  for over thirty years in the context of real-time obstacle avoidance. They were first in the seminal paper by Khatib \cite{khatib1986real}, and has since been developed further, beginning with computational methods \cite{barraquand1992numerical}. Of particular importance to this work, there has also been significant work in applying APFs to obstacle avoidance \cite{warren1989global,lee2003artificial}, including dynamic obstacles \cite{ge2002dynamic}. Work has also been been done on improving the behavior of artificial potential fields, particularly in dealing with undesirable oscillation behavior \cite{ren2006modified}. Finally, the search for effective methods for path planning using APFs has continued \cite{li2012efficient}, including application to UAV path planning \cite{chen2016uav}.

Control barrier functions (CBFs) were introduced recently \cite{ames2014control,ames17}, and serve as a method for providing safety guarantees of nonlinear systems via optimization-based controllers. They are commonly used in the safety-critical controls community due to their robustness \cite{xu2015robustness} and real-time performance capabilities, even for dynamic robots \cite{gurriet2018towards,nguyen20163d}.  They have, therefore, found application in a variety of domains: automotive safety \cite{jankovic2018robust}, robotics \cite{landi2019safety,cortez2019control}, multi-agent systems \cite{Wang17,glotfelter2017nonsmooth} and quadrotors \cite{wang2017safe}. See \cite{ames2019control} for a recent survey.  

%lindemann2019control

% on very complicated nonlinear systems like legged robots \cite{nguyen20163d}. 

% Ensuring safety via quadratic programming based control laws were popularized by \cite{ames17}, where safety constraints were incorporated via control barrier functions (CBFs). This was first applied to adaptive cruise control, and has since been utilized in a variety of application domains: automotive safety \cite{jankovic2018robust}, robotics \cite{cortez2020correct,cortez2019control,rauscher2016constrained} and multi-agent systems \cite{Wang17,lindemann2019control}. See \cite{ames2019control} for a recent survey. 

Given the historic use of artificial potential fields, and the recent popularity of control barrier functions, the natural question to ask is: \emph{How do control barrier functions compare to artificial potential fields?}

\begin{figure}
    \centering
    \includegraphics[width=\columnwidth,trim=0 400 0 0, clip]{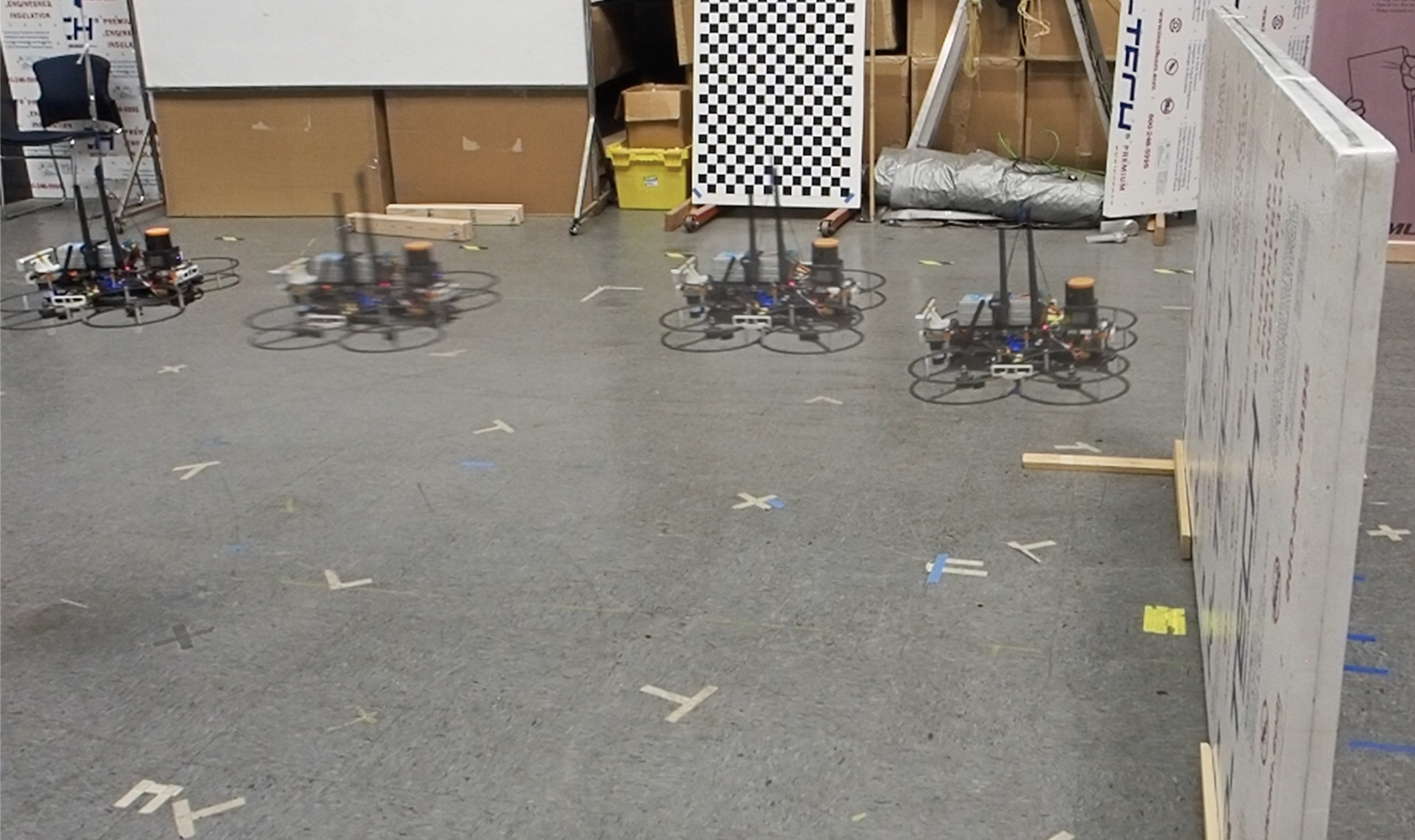}
    \caption{Quadrotor given a waypoint on the other side of wall, and stopping before colliding using control barrier functions.}
    \label{fig:intro_quad}
\end{figure}

% Since the main application of control barrier functions is to provide safety guarantees for the systems in which they are implemented, their use in the literature is generally heavily model-dependent. In contrast, we attempt to show in this work just how effectively control barrier functions can be utilized even without a model of the system. And while the provable safety guarantees are lost in this context, safety can be achieved practically with little-to-no tuning.

% We demonstrate this conjecture in simulations of a single integrator, a double integrator, and a quadrotor, and then showcase the robustness in the transition to hardware on a quadrotor. In both the simulation and hardware experiments of the quadrotor, the only tuning parameter of a control barrier function, $\alpha(\cdot)$, is kept at the constant $\alpha \equiv 1$, and the only model used is that of a single integrator, where the control input is the velocity of the mobile robot.

The major contribution of this work is a comparative analysis of CBFs and APFs---both theoretically and through simulation and experimental results on obstacle avoidance.  At a formal level, we establish that APFs can be used to synthesize a specific instance of a CBF, thereby showing that CBFs are more general than APFs.  Additionally, this translation results in beneficial properties: it (pointwise) optimally balances avoidance and goal attainment, is well defined if the system leaves the safe set, and allows one to generalize APFs to nonlinear control systems.   From a comparative perspective, we begin with simple examples to illustrate the beneficial properties of CBFs vs. APFs, followed by high-fidelity simulations of a quadrotor for different obstacle avoidance scenarios.  Finally, these same scenarios are carried out experimentally on a quadrotor with onboard sensing.  
%These comparisons suggest that CBFs outperform APFs. 
We conclude from these comparisons that CBFs outperform APFs in the context of providing smooth behaviors that are minimally invasive while guaranteeing obstacle avoidance.

% In the latter case, we show in simulation and on hardware that CBFs outperform APFs in the context of providing smooth behaviors that are minimally invasive while guaranteeing obstacle avoidance. 

% The major contribution of this work is a comparative analysis of CBFs and APFs---both theoretically and through simulation and experimental results on obstacle avoidance.  At a formal level, we establish that APFs can be used to synthesize a specific instance of a CBF, an APF-CBF, thereby showing that CBFs are more general than APFs.  Additionally, through this translation the APF-CBF enjoys additional beneficial properties: it more explicitly balances avoidance and goal attainment, is well defined if the system leaves the safe set, and allows one to generalize APFs to nonlinear control systems.   From a comparative perspective, we begin by showing that even in the case of a single integrator, CBFs have beneficial properties when compared to APFs.  This comparison is further explored via high-fidelity simulations of a quadrotor over 5 different obstacle avoidance scenarios.  Finally, these same scenarios are carried out experimentally on a quadrotor with onboard sensing.    

% The major contribution of this work, other than the comparison of CBFs and APFs, is the formulation of a general artificial potential field as a control barrier function. Not only does this result in a smoother behavior than the APF alone, but it highlights the broadness of control barrier functions as an obstacle avoidance method. 

The layout of the paper is as follows. Section \ref{sec:background} provides an overview of the APF and CBF methods, and illustrates their differences via an example of obstacle avoidance with single integrator dynamics. Section \ref{sec:pf_as_cbf} further details control barrier functions, and presents the main theoretic results of the paper including the fact that APFs are a special case of CBFs.  
Section \ref{sec:dynamics} presents simulation results with velocity-based controllers obtained from CBFs and APFs, which are compared in a series of scenarios.  Finally, Section \ref{sec:hardware} showcases the hardware results realizing APFs and CBFs on a quadrotor utilizing only onboard sensing and computation. 

% showcases how these model-free methods perform in the presence of real dynamics, in the form of a double integrator and a quadrotor tracking the desired velocities in simulation. Finally, Section \ref{sec:hardware} showcases the hardware results on a quadrotor utilizing only onboard sensing and computation.

\newpage

\section{Background \& Motivation}
\label{sec:background}

In this section, we give a brief background on artificial potential fields and control barrier functions, and illustrate their similarities and differences via an example associated with obstacle avoidance for a single integrator in the plane.  This comparison will be formalized in the next section, where more general forms of potential fields and control barrier functions will be considered.  Importantly, as the motivation in this section suggestions, we will see that control barrier functions are a generalization of potential fields. 

% More specifically, we consider the dynamical system:
% \begin{equation}
%     \dot{x} = v,
% \end{equation}
% with $x \in \R^2$ is the position, and $v \in \R^2$ is the velocity, which is also the input to the system.

\subsection{Artificial Potential Fields}

% Artificial Potential Fields, or APFs, have been utilized for over thirty years in the context of real-time obstacle avoidance.  They were first introduced in the seminal paper by Khatib \cite{khatib1986real}, and has since been utilized extensively \cite{?}.  A variety of potential functions can be utilized, but this section will look at the original formulation. A more general form, as well as another specific formulation, will be observed in the following sections.

We begin by considering artificial potential fields (APFs) in the setting of obstacle avoidance. 
A variety of potential functions can be utilized, but this section will look at the original formulation \cite{khatib1986real}.  
%  A more general form, as well as another specific formulation, will be observed in the following sections.
In this context, consider a control system described by a single integrator: 
\begin{eqnarray}
\label{eqn:singleint}
    \dot{x} = v,
\end{eqnarray}
with $x \in \R^n$ is the position, and $v \in \R^n$ is the velocity.  Here $v$ is viewed to be the control input to the system.  The goal is to synthesize a desired velocity profile that reaches a goal position while avoiding one or multiple obstacles.  The motivation for considering a single integrator is that the resulting behavior of this system can, for example, be utilized as desired velocity profiles for end-effector positions for a robot manipulator ($n = 3$) wherein classic Jacobian methods can be utilized
\cite{warren1989global}.  

%\cite{siciliano1990kinematic,xiang2010general,antonelli2006kinematic}.  

%\cite{siciliano1990kinematic,warren1989global}.  

% Consider a system that wishes to reach some goal position while avoiding obstacles. The general idea of potential fields is to move through space subject to an attractive force from the goal, while being repelled away from the obstacles.

To explicitly present artificial potential fields, per the original formulation in \cite{khatib1986real}, let $x_{\rm{goal}}$ the goal position.  This exerts an attractive potential to the system given by: 
\begin{equation}
\label{eq:khatib_attr}
    U_{\textrm{att}}(x) = \frac{1}{2}K_\textrm{att}\norm{x-x_\textrm{goal}}^2.
\end{equation}
% \begin{equation}
%     U_{\textrm{att}} = \frac{1}{2}K_\textrm{att}\norm{x-x_\textrm{goal}}^2.
% \end{equation}
Any obstacles in the area assert a repulsive potential, given by
\begin{equation}
\label{eq:khatib_rep}
    U_{\textrm{rep}}(x) = \begin{cases} 
      0 & \rho(x) > \rho_0 \\
      \frac{1}{2}K_\textrm{rep}\left( \frac{1}{\rho(x)} - \frac{1}{\rho_0}\right)^2 & \rho(x) \leq \rho_0
   \end{cases}
\end{equation}
where $\rho(x)$ is the distance to the obstacle or the distance from a safe region around the obstacle, e.g.: 
\begin{eqnarray}
\label{eqn:rho}
\rho (x) = \norm{x-x_\textrm{obs}} - D_{\textrm{obs}},
\end{eqnarray}
for $D_{\textrm{obs}} > 0$, and $\rho_0$ is the region of influence. The potential function is set to zero outside of this region to allow the attractive potential to dominate over large distances.

To obtain a feedback controller that pushes the system to the goal while avoiding obstacles, the attractive and repulsive potentials are combined and the gradient is taken: 
\begin{equation}
F_\textrm{APF}(x) = -\nabla U_\textrm{att}(x) -\nabla U_\textrm{rep}(x).
\end{equation}
with $\nabla U_{-}(x) = \frac{\partial U_{-}}{\partial x}(x)^T$ and:
\begin{align}
    \label{eq:khatib_attr_grad}
    \nabla U_\textrm{att}(x) &= K_{\textrm{att}}(x-x_{\textrm{goal}}) \\
    \nabla U_\textrm{rep}(x) &= \frac{K_{\textrm{rep}}}{\rho(x)^2}\left( \frac{1}{\rho(x)} - \frac{1}{\rho_0}\right)\frac{(x-x_{\textrm{goal}})}{\rho(x)}.
\end{align}
For a single integrator \eqref{eqn:singleint}, where one directly controls velocity, we simply apply this force as the velocity input: 
$$
\dot{x} = F_{\textrm{APF}}(x)
$$
yielding a gradient dynamical system with respect to the attractive and repulsive potentials. 

%with $v = F_{\textrm{APF}}$

\begin{example}
\label{ex:APF}
Consider a mobile robot modeled as a single integrator travelling in the plane: $n = 2$. The initial position is $x_0 = (0,0)$, and the goal position is $x_{\textrm{goal}} = (3,5)$. There are two obstacles, at $x_{O1} = (1,2)$ and $x_{O2} = (2.5,3)$, that the mobile robot must not come within 0.5 meters of these obstacles. 

The new distance function for obstacle $i$ is $\rho = \norm{x - x_{Oi}} - 0.5$. Using $K_{\textrm{att}} = K_\textrm{rep} = 1$, with varying values of $\rho_0$, we have the result shown in Figure \ref{fig:pf}(a).  As can be seen in this figure, the potential field works well at $\rho_0$ values of 1 and 0.25, but it gets stuck in a local minimum at $\rho_0 = 0.5$, and oscillations start to occur at $\rho_0 = 0.1$.  Potential fields, therefore, work well when properly tuned, but can suffer from oscillations---due to the interaction between the attractive and repulsive forces---when not tuned properly. 
\end{example}

\begin{figure*}[t]
    \centering
    \begin{subfigure}[t]{0.32\textwidth}
     \includegraphics[width=1\columnwidth,trim= 35 0 50 0,clip]{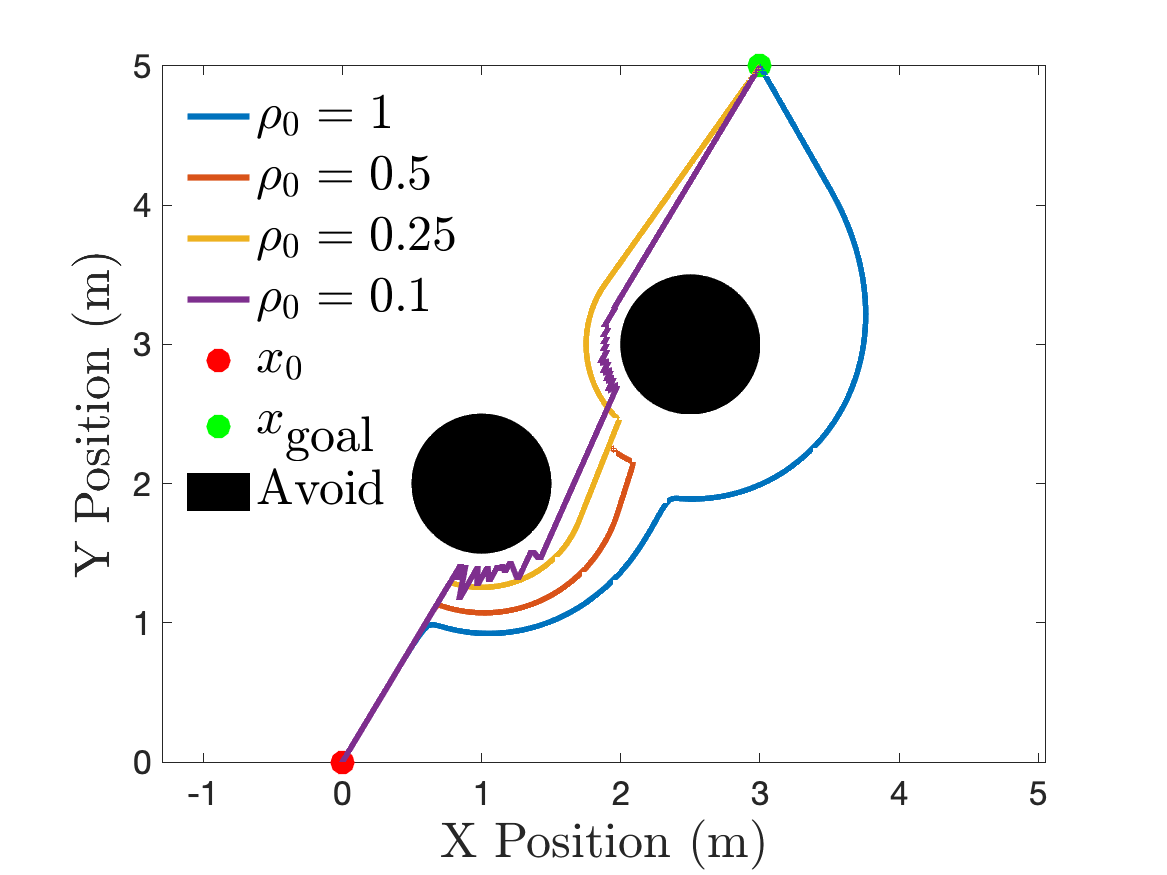}
    \caption{Artificial Potential Field}
    \end{subfigure}
    \begin{subfigure}[t]{0.32\textwidth}
    \includegraphics[width=1\columnwidth,trim= 35 0 50 0,clip]{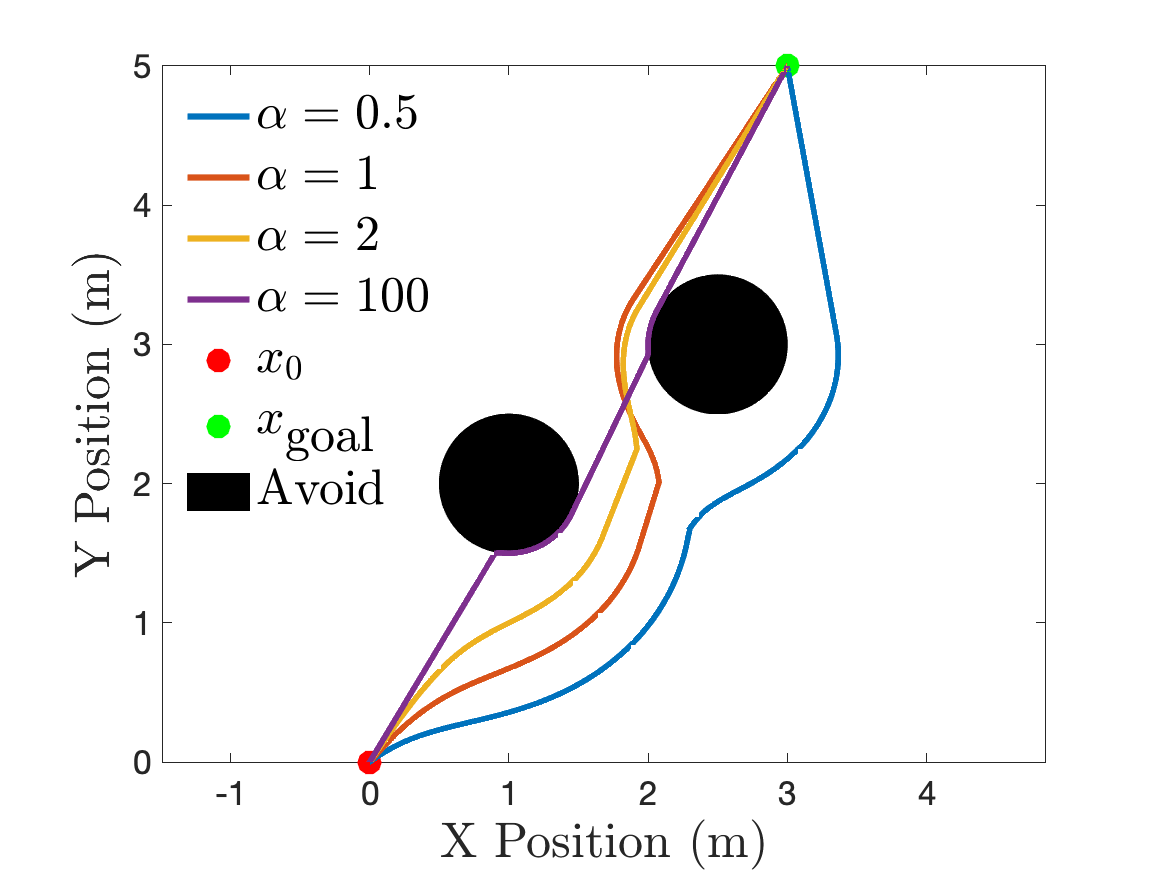}
    \caption{Control Barrier Function}
    \end{subfigure}
    \begin{subfigure}[t]{0.32\textwidth}
    \includegraphics[width=1\columnwidth,trim= 35 0 50 0,clip]{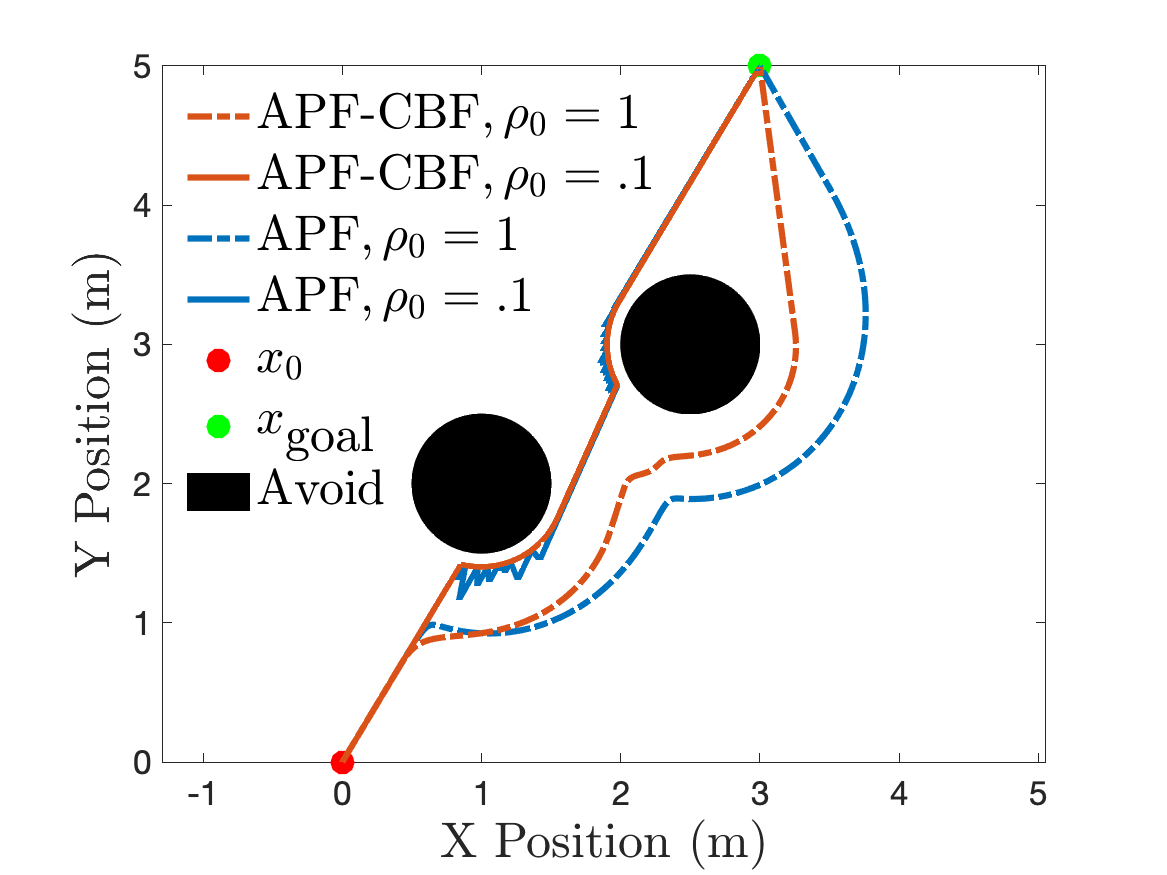}
    \caption{APF-based CBF}
    \end{subfigure}
    \caption{Comparison of potential fields and control barrier functions at various values of $\rho_0$ and $\alpha$}
    \label{fig:pf}
\end{figure*}

%\newpage

\subsection{Motivating Control Barrier Functions}
\label{sec:CBFmotivation}

To motivate control barrier functions, we again consider the the single integrator in \eqref{eqn:singleint}.  For this system, we wish to formulate safety-critical controller synthesis.  In this context, consider a safety function $h : \R^n \to \R$ defining a safe set: 
$$
\mathcal{S}  = \{ x \in \R^n ~ : ~ h(x) \geq 0\}.  
$$
That is, the system is ``safe'' when the function $h$ is positive.  One can view this set as the complement of the obstacles.  That is, one can utilize $\rho$ and define: 
$$
h(x) = \rho(x) =  \norm{x-x_\textrm{obs}} - D_{\textrm{obs}} \geq 0 
$$
where now the safe set $\mathcal{S}$ is the set we wished to render safe with the potential fields. 

% For example, $h(x) = x$ implies that the system should evolve in the set defined by $x \geq 0$, i.e., the positive real line. 

It is not possible to directly ensure safety of the system by simply checking on the positivity of $h$.  Therefore, we can instead consider a derivative condition on $h$ that can be checked instantaneously---this is analogous to potential fields wherein one looks at the gradient of the potential functions.  
This lead to the recent introduction of control barrier functions \cite{ames2014control,ames17} (see \cite{ames2019control} for a more detailed history).   The function $h$ is a \emph{control barrier function} if 
for all $x$ (in the domain of interest) there exists a $v$ such that:
\begin{align}
\label{eqn:hdotsimple}
\dot{h}(x,v) =  \nabla h(x)^T \dot{x} =
 \nabla h(x)^T  v \geq - \alpha h(x)
\end{align}
for $\alpha > 0$.  In this case, given an input $v(x)$ that satisfies this inequality, consider the closed loop system $\dot{x} = v(x)$ with solution $x(t)$ and initial condition $x(0) = x_0$ wherein, because $v(x)$ satisfies \eqref{eqn:hdotsimple} the set $\mathcal{S}$ is \emph{forward invariant}, i.e., \emph{safe}:
\begin{eqnarray}
\label{eqn:safetysimple}
x(0) = x_0 \in \mathcal{S} \qquad \Rightarrow \qquad 
x(t) \in \mathcal{S} \quad \forall ~ t \geq 0. 
\end{eqnarray}
Thus, the existence of a control barrier function implies safety. 
For example, one could verify that the potential field controller yields safe behavior by verifying that, for $h(x) = \rho(x)$, 
$$
\dot{h}(x,F_{\textrm{APF}}(x)) \geq - \alpha h(x). 
$$

Control barrier functions can be used to synthesize controllers that ensure safety, as defined in \eqref{eqn:safetysimple}, even given a desired velocity: $v_{\rm des}(x,t)$.  
Specifically, the input $v \in \R^n$ can be explicitly found that satisfies the inequality in \eqref{eqn:hdotsimple} if $\nabla h(x) \neq 0$, i.e., if $h$ has \emph{relative degree} $1$ subject to minimizing the difference between this input and and the desired input.  This can be framed as an optimization problem, and specifically a quadratic program (QP): 
\begin{align}
\label{eqn:QPsimple}
v^*(x,t) = \underset{v \in \R^n}{\operatorname{argmin}} & ~  ~ \| v - v_{\rm des}(x,t) \|^2 \\
\mathrm{s.t.} & ~  ~ \nabla h(x)^T v \geq - \alpha h(x), \nonumber
\end{align}
where $v^*(x,t)$ is the pointwise optimal controller.  This is an important and substantial divergence from potential fields in that gradients are no longer used for synthesis.  Rather, one can optimize over controllers that satisfy the safety constraint.  To see how this difference manifests itself, we return to Example \ref{ex:APF} but, instead, apply control barrier functions. 

\begin{example}
\label{ex:CBF}
Consider the same setup as Example \ref{ex:APF}. The desired velocity command is a simple $P$ controller on position,
\begin{equation}
    v_{\rm des}(x,t) = -K (x-x_{\textrm{goal}}),
\end{equation}
with $K = 1$. Note that this is equivalent to the $\nabla U_{\textrm{att}}$ from the previous example, as the attractive force functions as a P controller on position.
The control barrier function, as inspired by $\rho$ in \eqref{eqn:rho}, is given by: 
\begin{equation}
\label{eqn:hex}
h(x) = \underset{i \in \{1,2 \}}{\min} \norm{x - x_{Oi}} -  D_{\textrm{obs}},
\end{equation}
where $\nabla h(x) = \frac{x - x_{Oi}}{\norm{x - x_{Oi}}} $, for the closer obstacle $i$, and here we pick $ D_{\textrm{obs}} = 0.5$.  Note that, technically, this barrier function is non-smooth, but the methods from \cite{glotfelter2017nonsmooth} can be employed.  Practically, the obstacles are distances so no issues with continuity are encountered. 

The simulation is run for varying values of $\alpha$, and the results are shown in Figure \ref{fig:pf}(b). For all values of $\alpha$, the robot safely completes the mission and suffers from no oscillations.  Additionally, one can see the minimally invasive behavior of \eqref{eqn:QPsimple} in that the nominal trajectories to goal are modified to a much smaller degree when compared against potential fields. 
\end{example}

\section{Potential Fields as Control Barrier Functions}
\label{sec:pf_as_cbf}
% In this section, we show that potential fields can be implemented as a control barrier function, and we demonstrate the results on the single integrator in the plane. First, we begin by generalizing the potential functions.

In this section, we show the main result of this paper: that potential fields are a specific instance of control barrier functions.  This will be demonstrated by explicitly constructing a CBF from a potential field.  Importantly, this transformation results in additional beneficial properties that the original controller did not benefit from.  Thus, control barrier functions generalize potential fields.  Subsequent sections will practically demonstrate this in simulation and on hardware. 

% that potential fields can be implemented as a control barrier function, and we demonstrate the results on the single integrator in the plane. First, we begin by generalizing the potential functions.

\subsection{Control Barrier Functions}

We begin by introducing the general definition of control barrier functions, for which the constructions in Section \ref{sec:CBFmotivation} are a special case.  
In particular, the advantage of CBFs is that they can be applied to general nonlinear control systems.  Consider a nonlinear system of the form: 
\begin{align}
\label{eqn:controlsys}
    \dot{x} = f(x) + g(x) u.
\end{align}
with state $x \in \R^n$, input $u \in \R^m$, and $f,g : \R^n \to \R^n$ assumed to be Lipschitz continuous. 

\begin{definition}[\cite{ames17}]
\label{def:cbf}
Let $\S \subset \R^n$ be the set defined by a continuously differentiable function $h: \R^n \to \R$:
\begin{eqnarray}
\S & = & \{ x \in \R^n ~ : ~ h(x) \geq 0 \} , \nonumber\\
\partial \S & = & \{ x \in \R^n ~ : ~ h(x) = 0 \}, \nonumber\\
\mathrm{Int}(\S) & = & \{ x \in \R^n ~ : ~ h(x) > 0 \}. \nonumber
\end{eqnarray}
Then $h$ is a \textbf{control barrier function (CBF)} if $\nabla h(x)\neq 0$ for all $x\in\partial \S$ and there exists an \emph{extended class $\classK$ function} (\cite[Definition 2]{ames17})
$\alpha$ such that for all 
$ x \in \S$, $\exists u$ s.t. 
\begin{align}
\label{eqn:cbf:definition}
%\exists ~  u \quad \textrm{such that} \quad 
\underbrace{L_f h(x) + L_g h(x) u}_{\dot{h}(x,u)} \geq - \alpha(h(x)).
\end{align}
where $L_f h(x) : = \nabla h(x)^T f(x)$ and $L_g h(x): = \nabla h(x)^T g(x)$. 
\end{definition}

The main control barrier function result is that this class of functions give sufficient (and necessary) conditions on set invariance, i.e., safety of the system relative to $\mathcal{S}$. 

\begin{theorem}[\cite{ames17}]
\label{thm:cbf}
Given a control barrier function $h: \R^n \to \R$ together with the associated set $\mathcal{S}$, for any Lipschitz continuous controller satisfying: 
$$
\dot{h}(x,u(x)) = L_f h(x) + L_g h(x) u(x) \geq - \alpha(h(x)), 
$$
the set $\S$ is forward invariant, i.e,. safe.  Additionally, the set $\S$ is asymptotically stable. 
\end{theorem}

Since the constraint \eqref{eqn:cbf:definition} is affine in $u$, the above definition can be used to construct a quadratic program that functions as a safety filter, guaranteeing the safety of the system by enforcing it to stay inside of $\S$. This quadratic program is given by: 
\begin{align}
\label{eq:cbfqp}
\tag{CBF-QP}
u^*(x) = \underset{u \in \R^{m}}{\operatorname{argmin}} &  \quad  \| u - u_{\rm des}(x,t) \|^2  \\
\mathrm{s.t.} &  \quad L_f h(x) + L_g h(x) u \geq - \alpha(h(x))  \nonumber
\end{align}
Importantly, this QP has an explicit solution given by: 
\begin{eqnarray}
\label{eqn:minnormexplicit}
~  u^*(x,t)  =  u_{\rm des}(x,t) + u_{\rm safe}(x,t) 
\end{eqnarray}
where $u_{\rm safe}$ is added to $u_{\rm des}$ if the nominal controller would not keep the system safe, which is determined by the sign of $\Psi(x,t ; u_{\rm des}) := \dot{h}(x,u_{\rm des}(x,t))  + \alpha(h(x))$ via: 
\begin{gather}\label{eq:firstsolutiontoqp}
u_{\rm safe}(x,t) =  \left\{ 
\begin{array}{lcr}
- \frac{L_g h(x)^{T}}{L_g h(x) L_g h(x)^T} \Psi(x,t ;u_{\rm des}) & \mathrm{if~} \Psi < 0 \\
0 & \mathrm{if~} \Psi \geq 0
\end{array}
\right.  %\nonumber 
\end{gather}
Through this explicit form, one can begin to see the divergence between potential fields and CBFs: the condition statement indicates that safety is imposed only when necessary rather than always (via the addition of potentials).

\subsection{Potential Fields as a Special Case of CBFs}

We now present the main result of the paper: showing that from a potential field one obtains a control barrier function.  Importantly, one can use this understanding of potential fields to obtain additional beneficial properties of a given potential field: from safe set stability to the generalization of potential fields to general nonlinear control systems.   

\begin{definition}
\label{definition:pf_general}
%Consider a single integrator $\dot{x} = v$ as in \eqref{eqn:singleint}.  
Given a goal state, $x_{\rm goal} \in \R^n$, an \textbf{attractive potential} is a positive definite continuously differentiable function $U_{\rm att} : \R^n \to \R$, such that there exists $\underline{c}, \overline{c} > 0$ such that, $\forall ~ x \in \R^n$:
$$
\underline{c} \| x - x_{\rm goal} \|^2 \leq 
U_{\rm att}(x) \leq 
\overline{c} \| x - x_{\rm goal} \|^2. 
$$
Given an obstacle at $x_{\rm obst}$ and minimum distance $D_{\rm obst} > 0$, a \textbf{repulsive potential} is a continuously differentiable positive semi-definite function $U_{\rm rep}: \R^n \to \R$, strictly increasing, that ``blows up'' at the minimum distance:
\begin{align*}
\textrm{Positive semi-definite:} & \quad  U_{\rm rep}(x) \geq 0, \qquad  \\
\textrm{Strictly increasing:} & \quad  \nabla U_{\rm rep}(x) > 0 ~ \mathrm{if} ~
\| x - x_{\rm obst}\| \leq \rho_0 \\
\textrm{``Blows up'' at obstacle:} & \quad  \lim_{\| x - x_{\rm obst}\| \to D_{\rm obst}} U_{\rm rep}(x) = \infty. 
\end{align*}
An \textbf{artificial potential field}: $U(x) := U_{\rm att} (x) + U_{\rm rep}(x)$, yields a controller: 
$$
k(x) = - \nabla U(x)  \qquad \Rightarrow \qquad 
\dot{x} = - \nabla U(x). 
$$
\end{definition}

Note that the potential functions used in the previous section meet this definition.

% As mentioned previously, an artificial potential field combines these attractive and repulsive forces into a single function: $U(x) = U_{\rm att}(x) + U_{\rm rep}(x)$. One then synthesizes a controller: 
% $$
% k(x) = - \nabla U(x)  \qquad \Rightarrow \qquad 
% \dot{x} = - \nabla U(x), 
% $$
% where the gradient is just the Jacobian: $\nabla U(x) = J_U(x)$. 

\newsec{Main result.} The CBF paradigm includes potential fields as a special case.  In this case, rather than combining the attractive and repulsive potentials into a single function, we utilize the attractive potential as the desired velocity, and the repulsive potential as the control barrier function. 

\begin{theorem}
\label{thm:main}
Consider an artificial potential field with repulsive potential $U_{\rm{rep}}$ meeting Definition \ref{definition:pf_general}. The function: 
\begin{equation}
\label{eqn:hfromAPF}
    h(x) = \frac{1}{1 + U_{\rm rep}(x)} - \delta,
\end{equation}
with $\delta \in (0,\delta_0)$ a small constant, is a control barrier function for the single integrator: $\dot{x} = v$.  Additionally, given any feedback controller $v = k(x)$ satisfying: 
$$
\dot{h}(x,k(x)) = \nabla h(x)^T k(x) \geq - \alpha(h(x))
$$
the set 
$$
\mathcal{S} = \{ x \in \R^n  : h(x) \geq 0 \} \subset 
\{  x \in \R^n : \| x - x_{\rm obst}\| \geq D_{\rm obst}\}
$$
is forward invariant, i.e., safe, and asymptotically stable. 
\end{theorem}

\begin{remark}
Note that the set $\mathcal{S}$ depends on the choice of $\delta$ with:
$$
\lim_{\delta \to 0} \mathcal{S} = \{  x \in \R^n : \| x - x_{\rm obst}\| \geq D_{\rm obst}\}. 
$$
Thus, the smaller the choice of $\delta$, the closer the safe set, $\mathcal{S}$, to the complement of the obstacles.  Additionally, unlike potential fields, in the case when the system starts with an initial condition outside of $\mathcal{S}$ the CBF will both be well defined and the system will asymptotically stabilize back to $\mathcal{S}$. 
\end{remark}

\begin{proof}
Taking the gradient of $h(x)$ yields
\begin{equation*}
\nabla h(x) = - \frac{\nabla U_{\rm{rep}}(x)}{(1 + U_{\rm rep}(x))^2}. 
\end{equation*}
To be a control barrier function, we first require that $\nabla U_{\rm{rep}} \neq 0$ when $h(x) = 0$. When $h(x) = 0$, we have that
\begin{align*}
    \frac{1}{1+U_{\rm{rep}}(x)} - \delta = 0 
    \quad \implies  \quad 
    U_{\rm{rep}}(x) = \frac{1}{\delta} - 1
\end{align*}
Since $U_{\rm{rep}}$ ``blows up'' with proximity to obstacles, 
one can pick $\delta_0 > 0$ such that for all $\delta \in (0,\delta_0)$ it follows that $h(x) = 0$ implies that $\| x - x_{\rm obst} \| \leq \rho_0$. Therefore, by the strictly increasing assumption, $\nabla U_{\rm{rep}} \neq 0$. 

% inside of the region of influence $\rho_0$, and is zero outside of it, we have that $U_{\textrm{rep}} > 0 \implies \nabla U_{\textrm{rep}} \neq 0$. Thus, can use any $\delta \in (0,1)$.

For the single integrator dynamics, the CBF requirement given in Equation \eqref{eqn:cbf:definition} is trivially met, as $L_f h(x) = 0$ and $L_g h(x) =  \nabla h(x)^T v$, so there always exists a velocity such that  $- \frac{\nabla U_\textrm{rep}v}{(1 + U_{\rm rep}(x))^2} \geq - \alpha(h(x))$.  Therefore, $h$ is a CBF and the remaining statements follow from Theorem \ref{thm:cbf}. 
\end{proof}

\newsec{Controller synthesis.}
% Now, by using $- \nabla U_\textrm{att}(x)$ as the desired velocity, we have
The advantage of CBFs is that they allow for controller synthesis where the attractive and repulsive potentials are combined in a pointwise optimal fashion. 
Specifically, using $- \nabla U_\textrm{att}(x)$ as the desired velocity, we have:
\begin{align}
\label{eqn:QP_CBF_PF}
v^*(x) = \underset{v \in \R^n}{\operatorname{argmin}} & ~  ~ \| v + 
\nabla U_{\rm att}(x) \|^2 \\
\mathrm{s.t.} & ~  ~  - \frac{\nabla U_{\rm{rep}}^T v}{(1 + U_{\rm rep}(x))^2} \geq - \alpha(h(x)). \nonumber
\end{align}
Importantly, this controller has an explicit solution: 
\begin{gather}
\label{eqn:exp_CBF_PF}
\begin{split}
 ~  v^*(x)  = -\nabla U_{\rm{att}}(x) + 
%\qquad 
\left\{ 
\begin{array}{lcr}
- \frac{\nabla h(x)}{\nabla h(x)^T \nabla h(x)} \Psi(x ; U_{\rm{att}}) 
&  \hspace{-.2cm}  \mathrm{if~} \Psi < 0 \\
0 & \hspace{-.2cm}   \mathrm{if~} \Psi  \geq 0
\end{array}
\right.  
\end{split}
\end{gather}
for $\Psi = \Psi(x ; U_{\rm{att}}) := - \nabla h(x)^T \nabla U_{\rm{att}}(x)  + \alpha(h(x))$.

Note the parallels between this function and the original artificial potential field, where the repulsive potential only plays a difference when within a certain radius $\rho_0$. Now, the attractive potential is used unless $-\nabla h(x)^T \nabla U_{\rm{att}}  + \alpha(h(x)) < 0$, in which case the CBF minimally alters the velocity inputs in order to maintain safety.

\begin{example}
Consider the APF given in Example \ref{ex:APF}. The repulsive potential is given by \eqref{eq:khatib_rep}, which is used to make the barrier function \eqref{eqn:hfromAPF}, with value $\delta = 0.001$. The attractive potential given by \eqref{eq:khatib_rep} with gradient \eqref{eq:khatib_attr_grad} is used in the desired velocity controller. 

By applying the APF-CBF QP given in \eqref{eqn:QP_CBF_PF} or the explicit solution \eqref{eqn:exp_CBF_PF}, the path shown in \ref{fig:pf}(c) is obtained for tuning parameters $K_{\textrm{rep}} = K_{\textrm{att}} = \alpha = 1$. 
One can see improved performance for the APF-CBF QP obtained from the APF when compared against the original APF, wherein the APF-CBF QP gets closer to the obstacles with fewer oscillations.  Moreover, when the robot has passed the obstacle and is moving towards the goal, the APF-CBF converges more quickly to the desired path.
\end{example}

\newsec{Extending APFs to nonlinear systems.}  It is important to note that the connection between APFs and CBFs allows for potential fields to be generalized to a nonlinear setting with ease. 
In particular, since CBFs are defined general nonlinear control systems, as in \eqref{eqn:controlsys}, we can use the instantiation of APFs as CBFs to easy extend APFs to a nonlinear setting.  

\begin{proposition}
\label{prop:nonlinear}
Consider a nonlinear control system of the form: $\dot{x} = f(x) + g(x) u$.  Assume the existence of a potential field as given in Definition \ref{definition:pf_general} with the associated safety constraint, $h$, given in \eqref{eqn:hfromAPF}.  If the repulsive potential, $U_{\rm rep}$, satisfies the CBF condition: 
$$
\nabla U_{\rm{rep}}(x)^T g(x) = 0 \quad \Rightarrow \quad
-\frac{\nabla U_{\rm{rep}}(x)^T f(x) }{(1 + U_{\rm rep}(x))^2}
\geq - \alpha(h(x)) ,
% \nabla U_\textrm{rep}(x)^T f(x) \geq 
% - \alpha(h(x)) (1 + U_{\rm rep}(x))^2
$$
for some extended class $\mathcal{K}$ function $\alpha$ then the controller:
\begin{align}
\label{eqn:QP_CBF_PF2}
u^*(x) = \underset{v \in \R^n}{\operatorname{argmin}} & ~  ~ \| u + 
\nabla U_{\rm att}(x) \|^2 \\
\mathrm{s.t.} & ~  ~  - \frac{\nabla U_{\rm{rep}}^T (f(x) + g(x) u) }{(1 + U_{\rm rep}(x))^2} \geq - \alpha(h(x)). \nonumber
\end{align}
renders the set $\{ x \in \R^n : \| x - x_{\rm obst} \|  \geq D_{\rm obst} \}$ forward invariant, i.e., a controller that ensures safety. 
\end{proposition}

\begin{proof}
The CBF condition is simply: $L_g h(x) = 0$ implies that $L_f h(x) \geq - \alpha (h(x))$.  One can verify that this implies that \eqref{eq:firstsolutiontoqp} is well defined, thus \eqref{eqn:cbf:definition} is satisfied and $h$ is a CBF.  The result then follows by combining Theorem \ref{thm:cbf} with Theorem \ref{thm:main}. 
\end{proof}

\section{Application to Quadrotors}
\label{sec:dynamics}

In the previous sections, our examples involved single integrators where the velocity commands from the APFs and the CBFs are tracked perfectly by the system.  This section will consider the application of the ideas presented on systems with non-trivial dynamics and, specifically, quadrotors. While we could apply the previous results, e.g., Proposition \ref{prop:nonlinear}, this would require both detailed model information and the ability to achieve torque control.  On quadrotors this is difficult in practice.  We thus describe the process of translating the formal results to practice via velocity-based tracking controllers. 

% However, this is never the case for actual systems, such as robotic systems with (often bounded) force and torque inputs. And while many existing papers show safety guarantees for dynamical systems using control barrier functions, they all require some amount of model information. In this section, we will observe the performance of the CBFs that filter the velocity using only the single integrator model, and compare this to the APFs, when applied to robotic systems that attempt to track these velocities.

%\subsection{Motivation: Double Integrator}

\newsec{Motivation: double integrator.}  To illustrate the issue with utilizing model-free collision avoidance on systems with non-trivial dynamics, we will compare the performance of the same artificial potential field and control barrier function described in Examples \ref{ex:APF} and \ref{ex:CBF}, but applied to a double integrator: 
\begin{gather}
\begin{split}
    \dot{x} &= v\\
    \dot{v} &= u
\end{split}
\end{gather}
The velocity outputs of the safety filters, denoted $v^*(x)$, are now tracked by the velocity-based controller:
\begin{equation}
\label{eqn:uvelbased}
    u(x,v) = -K (v - v^*(x)). 
\end{equation}
% The artificial potential field is run identically to that in Example \ref{ex:APF}, and the robot path is shown in Figure \ref{fig:PF_DI}. 

\begin{figure}[b]
        \centering
    \includegraphics[width=.8\columnwidth,trim= 50 0 50 0,clip]{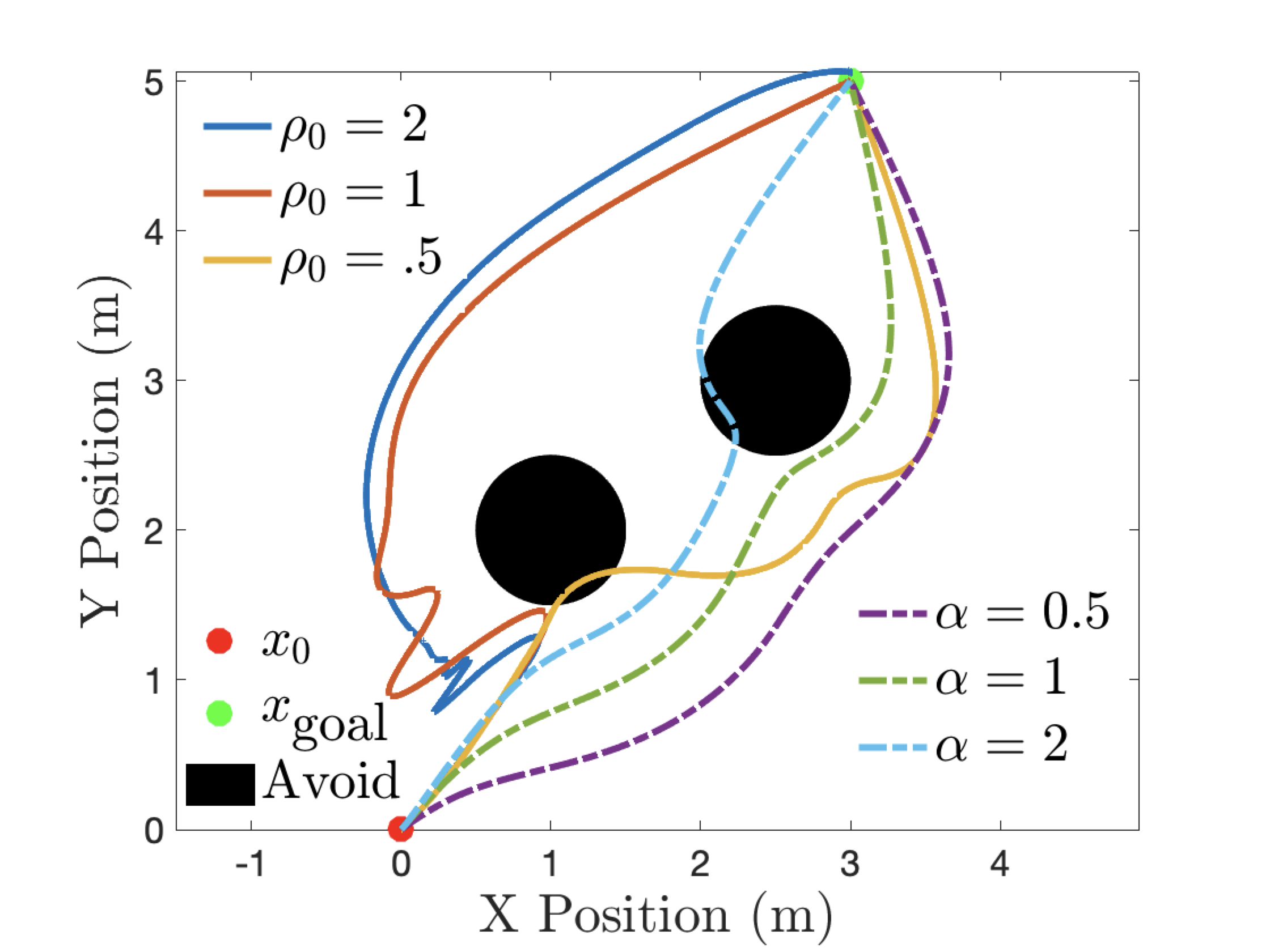}
    \caption{Potential Fields and CBFs on a double integrator.}
    \label{fig:PF_DI}
\end{figure}

The velocity-based controller is implemented for $v^*$ obtained from APFs and CBFs (per Examples \ref{ex:APF} and \ref{ex:CBF}). 
For the APF, the $\rho_0$ values of 1 and 2, the APF is able to keep the system safe while reaching the goal. However, large oscillations occur while approaching the first obstacle. The oscillations are not present with $\rho_0$ equal to 0.5, but safety is no longer maintained. Eliminating these oscillations and maintaining safety is possible, but would require additional tuning.  
% requires more extensive modification of the repulsive gain $K_{\textrm{rep}}$ as well $\rho_0$, or utilizing a different repulsive potential.
For CBFs, safety could be trivially guaranteed by utilizing the double integrator dynamics in the CBF-QP, but we instead utilize the velocity-based controller from Example \ref{ex:CBF}. Safety is maintained for $\alpha$ values of 0.5 and 1, but it is not maintained for $\alpha = 2$. This biggest difference between the performance of the CBF and the APF is the lack of oscillations when approaching the obstacles, with CBFs resulting in smoother controllers.  This is a trend that will be seen in simulation and experimentally on the quadrotor. 

\begin{figure*}[t]
    \centering
    \begin{subfigure}[t]{0.19\textwidth}
    \includegraphics[width=1\columnwidth,trim=0 0 0 20,clip]{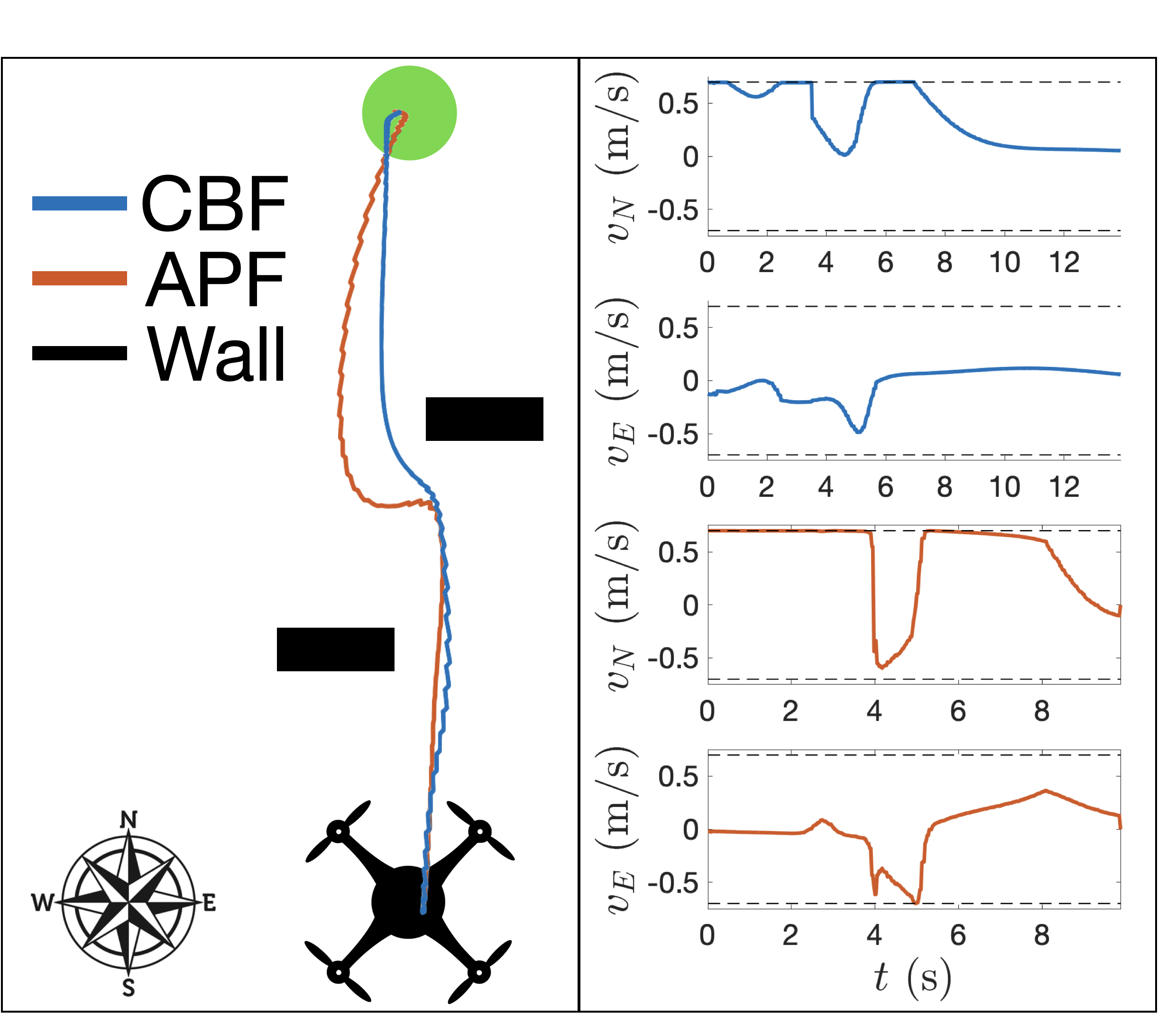}
    \end{subfigure}
    \begin{subfigure}[t]{0.19\textwidth}
    \includegraphics[width=1\columnwidth,trim=0 0 0 20,clip]{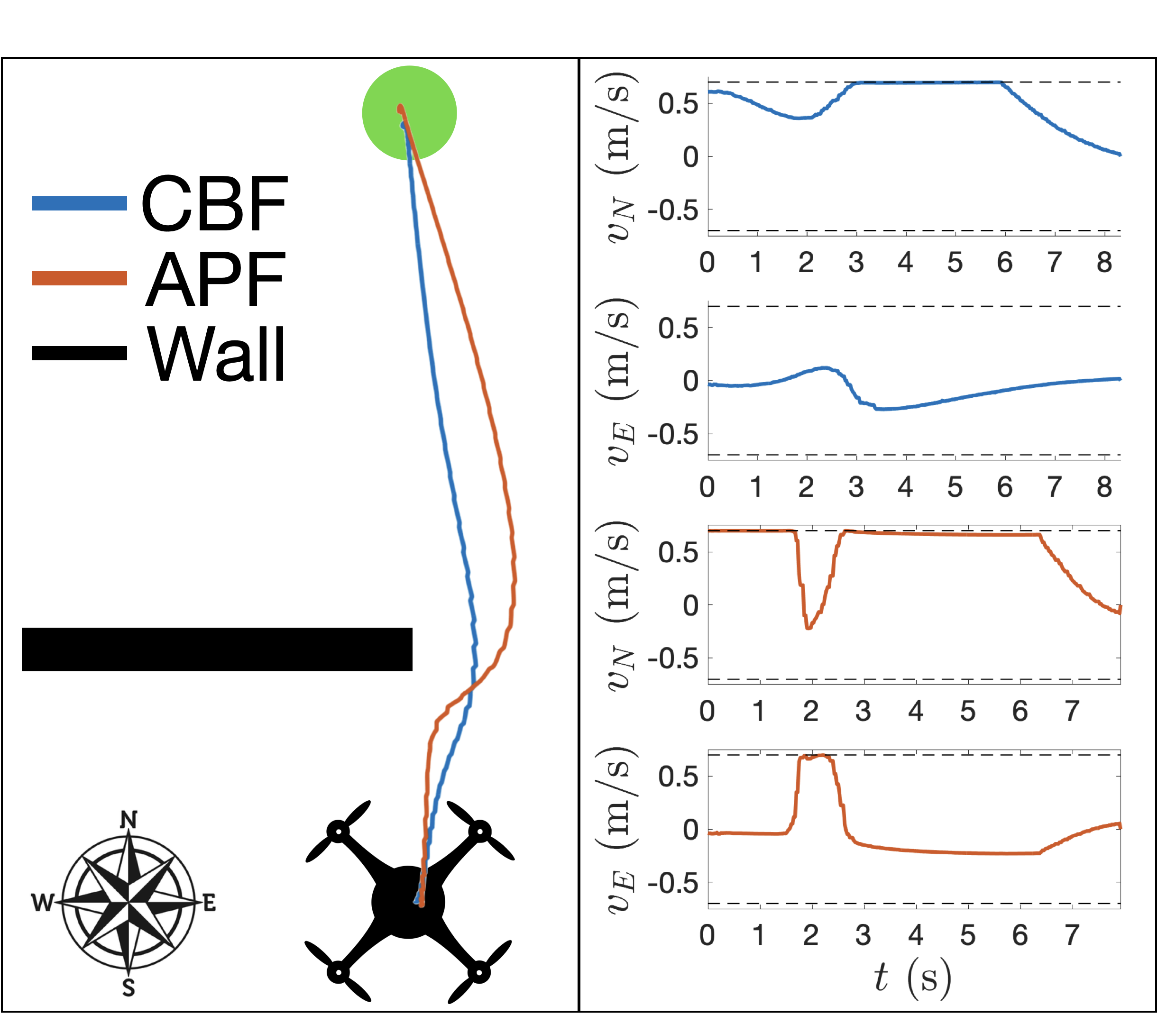}
    \end{subfigure}
    \begin{subfigure}[t]{0.19\textwidth}
    \includegraphics[width=1\columnwidth,trim=0 0 0 20,clip]{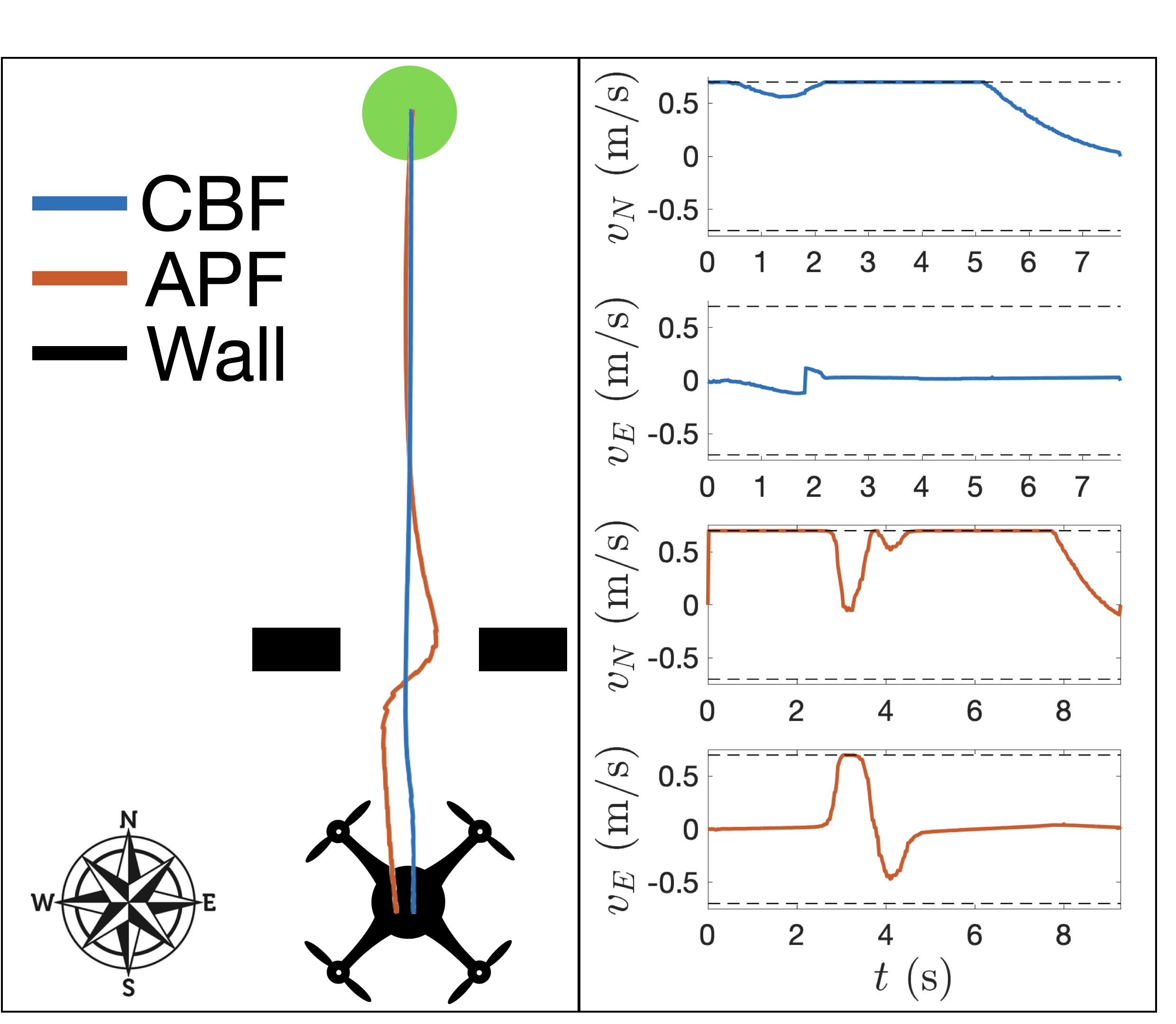}
    \end{subfigure}
    \begin{subfigure}[t]{0.19\textwidth}
    \includegraphics[width=1\columnwidth,trim=0 0 0 20,clip]{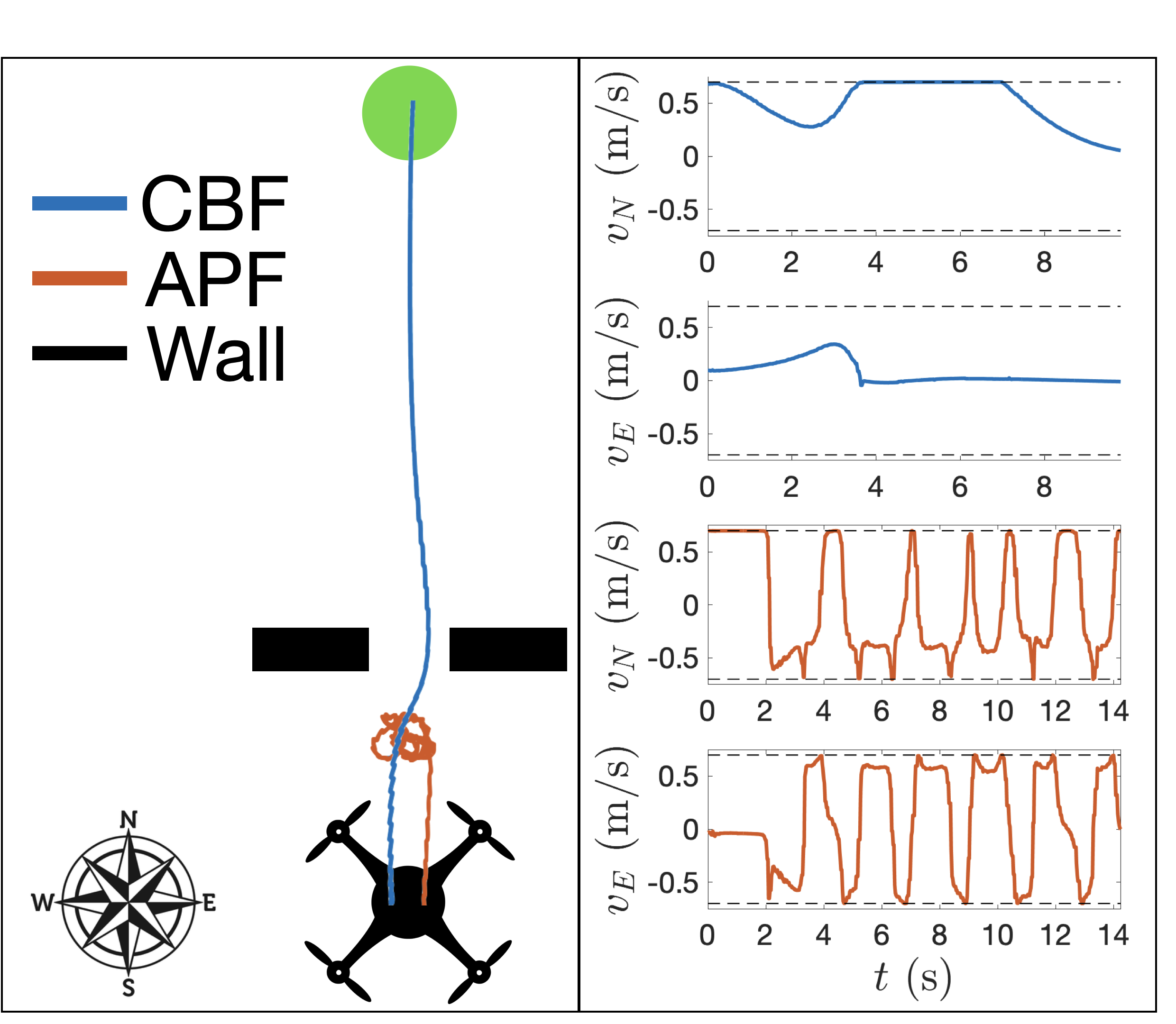}
    \end{subfigure}
    \begin{subfigure}[t]{0.19\textwidth}
    \includegraphics[width=1\columnwidth,trim=0 0 0 20,clip]{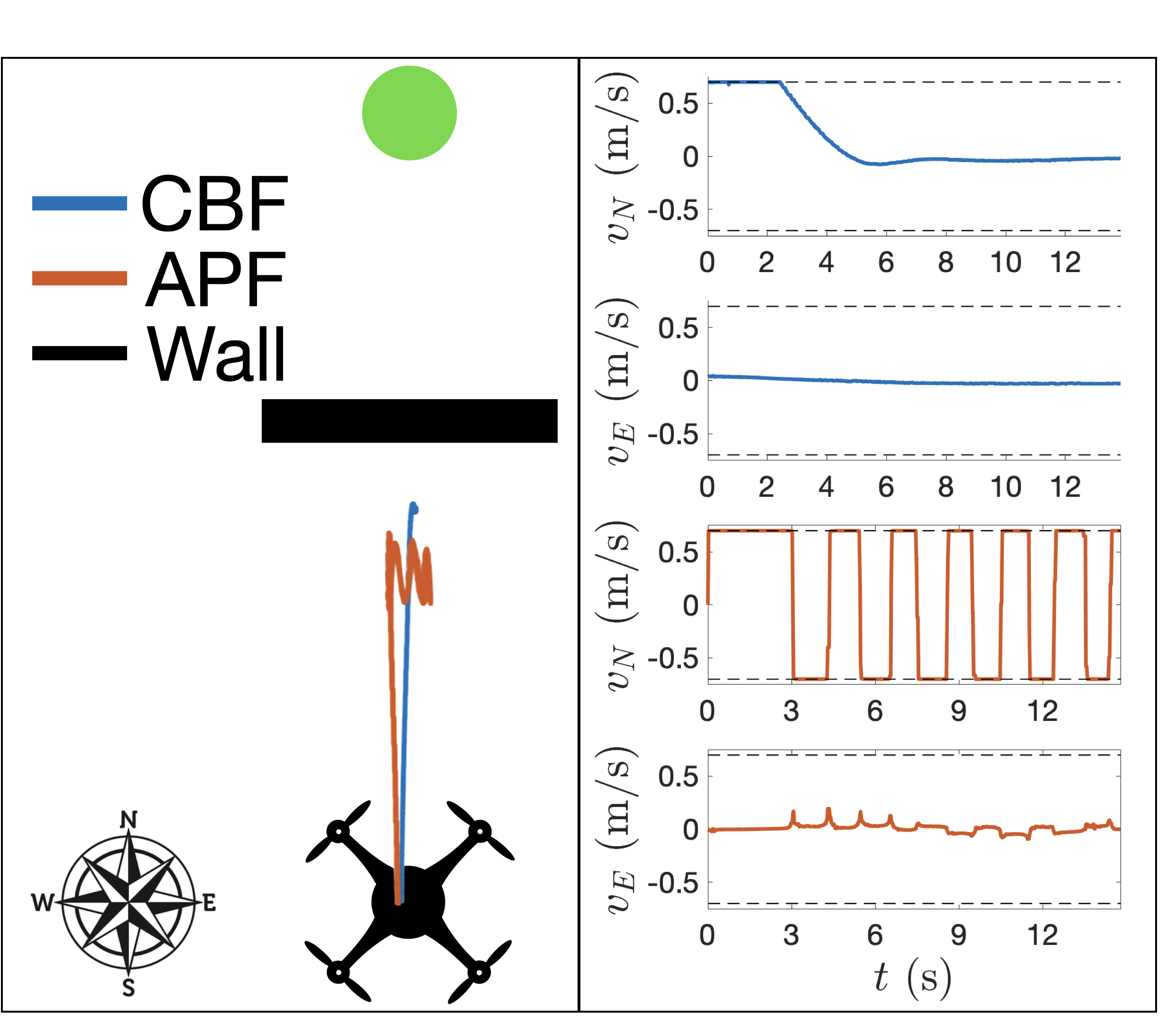}
    \end{subfigure}
    \caption{Simulation results for the quadrotor with APFs and CBFs for the five scenarios considered in this paper.}
    \label{fig:quad_sim}
\end{figure*}

%\newpage

\noindent \phantom{A} \vspace{-11.9 pt} \linebreak 
% \vspace{-8 pt}
\textbf{Application to Quadrotors in Simulation.}
To provide a more realistic comparison of APFs and CBFs realized as velocity-based controllers \eqref{eqn:uvelbased}, we consider their application to quadrotors in the context of obstacle avoidance.  In particular, we compare the APF and CBF velocity-based controllers in a high-fidelity simulation environment based on the physical hardware that will be detailed in the next section (see Fig. \ref{fig:hardware_setup}). The dynamics and low-level control is provided by the ArduPilot SITL simulator, and the velocity commands are produced and filtered in ROS nodes using simulated LIDAR sensor data from Gazebo.   

% test a more modern potential field formulation that has been tuned for a quadrotor in simulation. The environment is a high-fidelity simulation environment based on the physical system that will be detailed in the following section. The dynamics and low-level control is provided by the ArduPilot SITL simulator, and the velocity commands are produced and filtered in ROS nodes using simulated LIDAR sensor data from Gazebo.

In the context of the APF velocity-based controller, we will utilize the general form given in Example \ref{ex:APF}, with the same attractive potential as given in \eqref{eq:khatib_attr}.  The repulsive potential is replaced with a more modern potential field formulation that has been tuned for a quadrotor in simulation---this was done to avoid the repulsive potential from becoming ill-posed when realized in practice.  In particular, the repulsive potential is taken from \cite{gazi2007aggregation} yielding the repulsive force:
\begin{equation}
    F_{\textrm{rep}} = (x-x_{Oi})K_{\textrm{rep}}\exp \left(-\frac{\rho^2}{\rho_0}\right)
\end{equation}
with $\rho = \norm{x - x_{Oi}}$. The values of $K_{\textrm{rep}}$ and ${\rho_0}$ were tuned until oscillations vanished in practical cases, and safety was achieved, but optimized such that they do not affect flight when collision is unlikely.

% For the CBF-based velocity controller, we use a formulation identical to Example \ref{ex:CBF}. In particular, the barrier function is given as in \eqref{eqn:hex}, with $D_{\rm obs} = 0.3$: \begin{equation}
%     h(x) = \underset{i \in \{1,2 \}}{\min} \norm{x - x_{h,i}} - 0.3,
% \end{equation}
% where $x_{h,i}$ is the $i^{\rm th}$ point of the simulated laser scan, and $0.3$ represents the minimum distance that the drone must maintain from the obstacles. 
% Utilizing this with the single integrator dynamics results in the control law in \eqref{eqn:QPsimple} which is passed to \eqref{eqn:uvelbased}. 
% The value of $\alpha$ kept to 1, to ensure that no tuning is performed to improve the results.

For the CBF-based velocity controller, we use a formulation identical to Example \ref{ex:CBF}. In particular, the barrier function is given as in \eqref{eqn:hex},
where $x_{Oi}$ is the $i^{\rm th}$ point of the simulated laser scan, and $D_{\rm obs} = 0.3$ represents the minimum distance that the drone must maintain from the obstacles. 
Utilizing this with the single integrator dynamics results in the control law in \eqref{eqn:QPsimple} which is passed to a low-level velocity tracking controller. 
The value of $\alpha$ kept to 1, to ensure that no tuning is performed to improve the results.

%  with $D_{\rm obs} = 0.3$ 
%
% The formulation for the control barrier function is identical to that in the previous examples. Again, the dynamics are given by
% \begin{equation}
%     \dot{x} = v,
% \end{equation}
% with the velocity $v$ being the control input. The control barrier function is 
% \begin{equation}
%     h(x) = \underset{i \in \{1,2 \}}{\min} \norm{x - x_{h,i}} - 0.3,
% \end{equation}
% where $x_{h,i}$ is the $i^{\rm th}$ point of the simulated laser scan, and $0.3$ represents the minimum distance that the drone must maintain from the obstacles. The value of $\alpha$ kept to 1, to ensure that no tuning is performed to improve the results.

For each experiment, the quadrotor is given a waypoint 5m ahead in the x-direction. The five tests are described as follows, in order of difficulty for the collision avoidance algorithms:
(1) in between the quadrotor and the goal, two obstacles are placed that are offset from the center of the path, but close enough to effect the drone. (2) A single obstacle is placed such that the edge of the obstacle aligns with the center of the drone. This is done to ensure that the drone is able to find a path around it, but to strongly obstruct the drone. (3) Two obstacles are placed with edges 1 m apart, to mimic a 1 m wide doorway, and the drone has to fly through this to reach the goal.  (4) The doorway from the previous test is reduced to 0.7 m in width. (5) In between the starting position and the goal is a large wall that the drone is not able to pass. This is to test the oscillations that may occur when running directly at an obstacle that is in front of the goal.

% \begin{itemize}
%     \item In between the quadrotor and the goal, two obstacles are placed that are offset from the center of the path, but close enough to effect the drone.
%     \item A single obstacle is placed such that the edge of the obstacle aligns with the center of the drone. This is done to ensure that the drone is able to find a path around it, but to strongly obstruct the drone.
%     \item Two obstacles are placed with edges 1 m apart, to mimic a 1 m wide doorway, and the drone has to fly through this to reach the goal.
%     \item The doorway from the previous test is reduced to 0.7 m in width.
%     \item In between the starting position and the goal is a large wall that the drone is not able to pass. This is to test the oscillations that may occur when running directly at an obstacle that is in front of the goal.
% \end{itemize}

Each of the five tests are run for both the artificial potential field, as well as the control barrier function. The setups and paths are shown in Figure \ref{fig:quad_sim}, along with the velocities. Oscillations do not occur for the CBF velocity-based controller, and only occur for the APF based controller situations where the drone is unable to get to the goal, e.g., the wall and the 0.7 m doorway.  In all cases, the CBF is able to get closer to the obstacles while staying safe due to its pointwise optimally. 

\section{Experimental Results}
\label{sec:hardware}

In this section, we compare the performance of the two methods on a quadrotor. Localization and obstacle detection are done entirely onboard, using onboard sensors, making this a very practical comparison for real-world use.  Practically, we realize the APF and CBF controllers as outlined in Sect. \ref{sec:dynamics}.  

\newsec{Hardware setup.}
The quadrotor used for experiments is shown in Figure \ref{fig:hardware_setup}. It consists of a Lumenier Defender frame, four T-Motor F40 PRO II 1600 KV brushless motors, a Lumenier 50A 4-in-1 ESC, a mRobotics Pixracer R15 autopilot, a T265 RealSense camera, a Intel NUC i7 onboard computer, and a Hokuyo UST-10LX LIDAR. 
The Hokuyo UST-10LX 2D LIDAR gives 1080 points in front of the quadrotor in a 270$^\circ$ field of view along the XY plane. Google's Cartographer SLAM package was used with the Hokuyo LIDAR and the RealSense camera for localization. Additionally, the Hokuyo LIDAR is used for obstacle detection and avoidance. 
An Intel NUC i7 onboard computer is used to run the ROS nodes that perform the localization and collision avoidance. The high-level velocity commands are passed from the onboard computer to the Pixracer flight controller. The flight controller utilizes a cascading PID control structure of velocity, acceleration, attitude, and angular rate.

\newsec{Simulation vs hardware results.}
The same five tests described in the previous section (see Fig. \ref{fig:quad_sim}) were implemented on the hardware, and the results are shown in Fig. \ref{fig:all_tests}. The only difference in the setup was that the drone is now commanded to yaw $180^{\circ}$ and return along the same path, in order to maximize the amount of data for the analysis.
While the hardware results are similar to simulation for CBFs, the artificial potential field suffers from significantly more oscillations than in the high-fidelity simulation environment. This suggests that the CBF implementation is more robust to model uncertainty and noise, as the APF would need to be tuned again to eliminate oscillations due to the differences between the simulation and reality.  Finally, APFs fail to reach the goal in the case of the narrow door, while the CBFs succeeds. Thus, the CBFs outperformed APFs on hardware. 

\begin{figure}[t]
    \centering
    \includegraphics[width=.9\columnwidth,trim=0 0 0 80   ,clip]{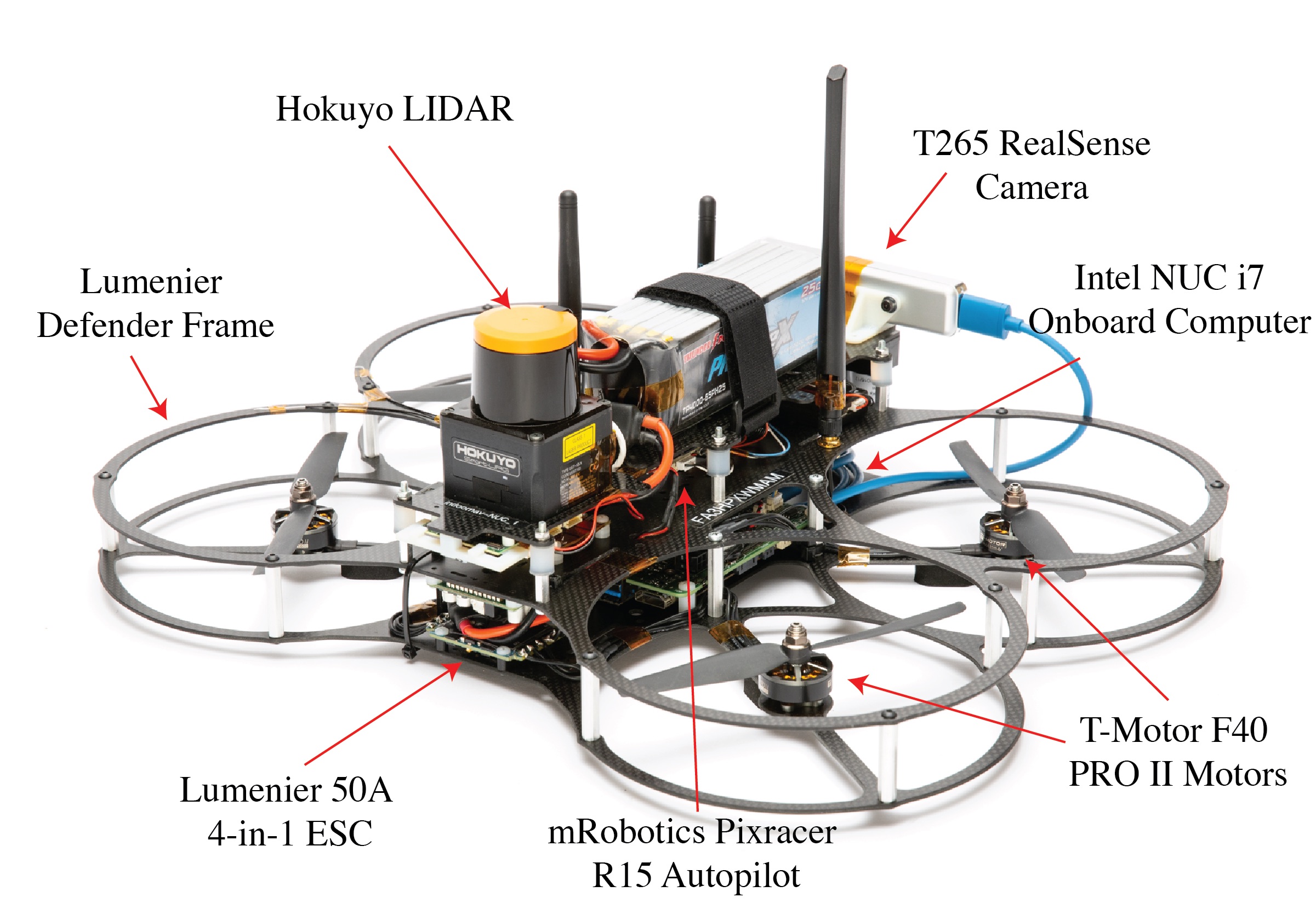}
    \caption{The quadrotor used in the experiments.}
    \label{fig:hardware_setup}
\end{figure}

\newsec{Manual flight.}
In addition to replicating the tests performed in simulation, the drone was also piloted manually through a more complicated obstacle course, to show the robustness of the control barrier function to sporadic user inputs. The pilot is trying to run the system towards the obstacles, but the CBF prevents the drone from crashing. The results are shown in Figure \ref{fig:manual} where it can be seen that the quadrotor gets close to the wall, but safety is maintained. % encoded by positively of the CBF.

\section{Conclusions}

In this paper, we showed that control barrier functions offer a viable, and arguably improved, alternative to artificial potential fields for real-time obstacle avoidance. It was shown that artificial potential fields can be formulated as control barrier functions, and the resulting behavior is smoother than the APF alone. CBFs were then implemented on a quadrotor via velocity-based control in simulation and on hardware, with no tuning whatsoever, and the results outperformed the existing (expertly tuned) APF-based collision avoidance system. 
Future work includes extending the model-based formal guarantees presented in this paper to the velocity-based model-free controller implemented in practice, along with further hardware demonstrations on dynamic robotic systems.  

% Future work includes formalizing safety guarantees for these model-free control barrier functions, utilizing bounds on the tracking error of the underlying controller, and further extending the APF-based CBF for general nonlinear systems.

\begin{figure}[t]
    \centering
    \includegraphics[width=.9\columnwidth,trim=0 0 0 90,clip]{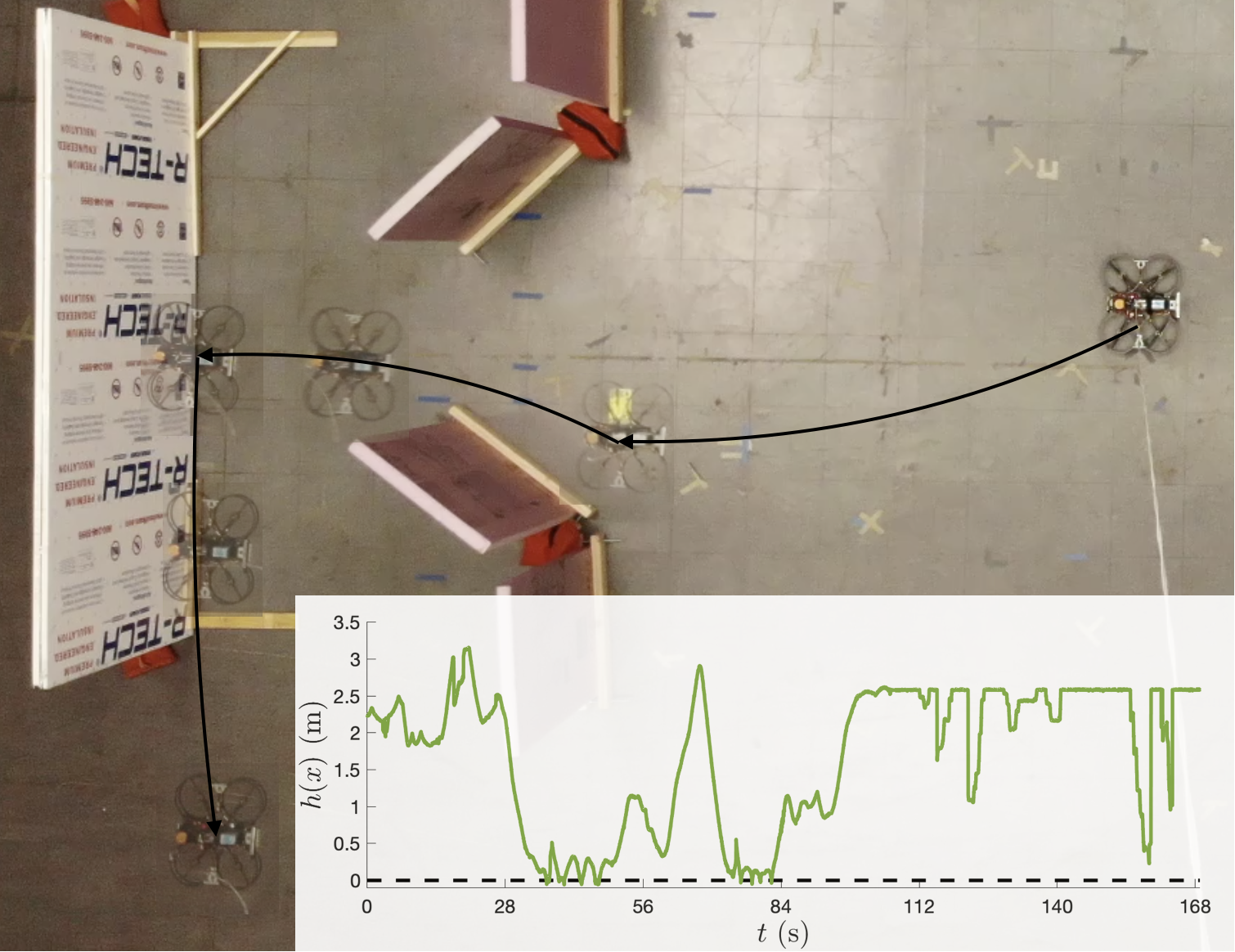}
    \caption{Partial trajectory and CBF data from the manual flight.}
    \label{fig:manual}
\end{figure}

\begin{figure*}[t]
    \centering
    \includegraphics[width=\textwidth,trim=0 0 0 0, clip]{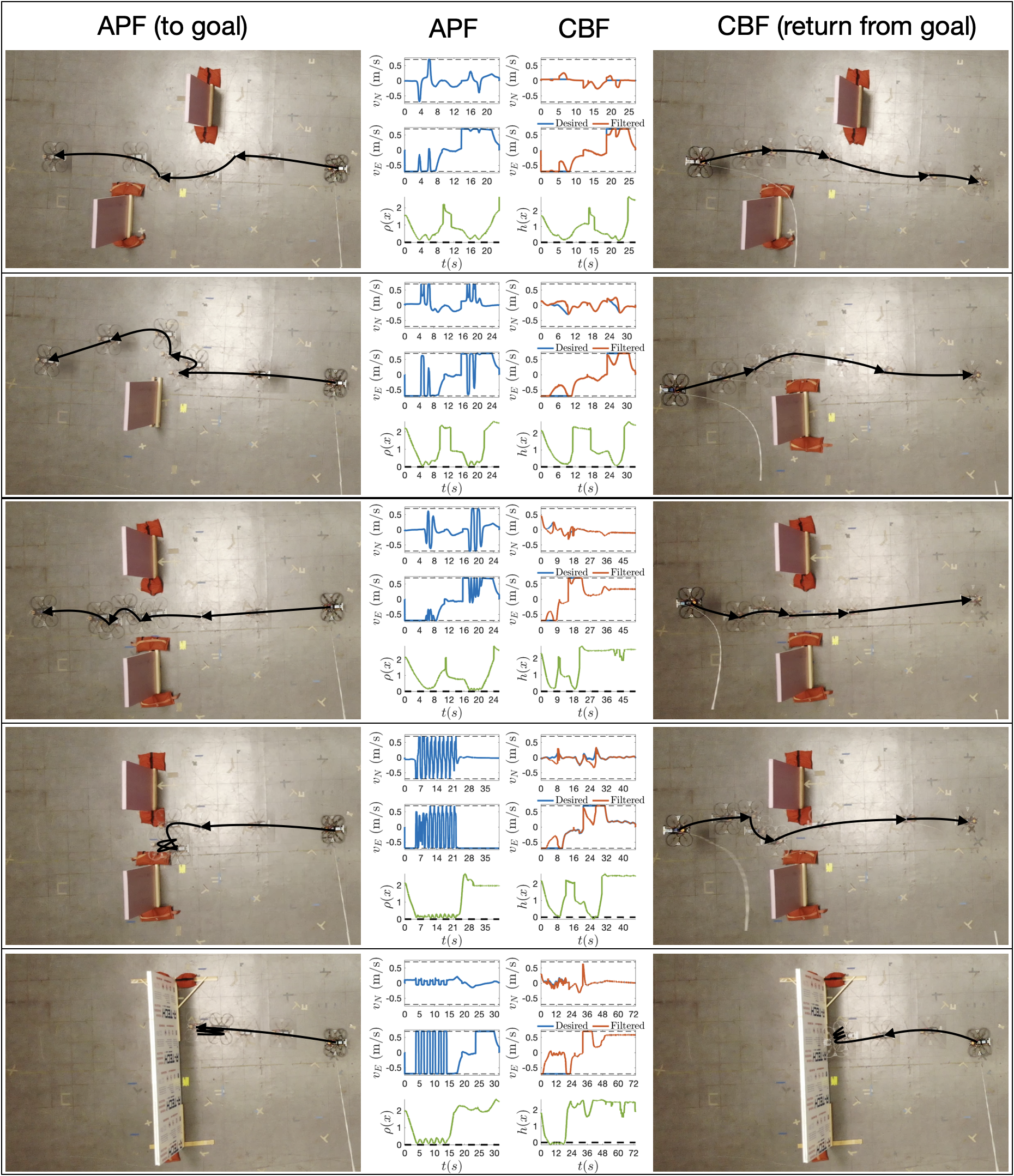}
    \caption{Hardware results for the control barrier function and artificial potential field. Note that the APF showcases the drone as it goes to the goal, while the CBF is pictured returning to the start position. Due to symmetry of the setup, these trips almost identical for each method, and the plots contain data for the round trip. The video found here \cite{vimeo_quad} shows both directions for each of the experiments from three different angles. The first two columns show the velocities in the north and east directions, and the final column shows the value of $\rho(x) - 0.3$ for the APF on the left, and $h(x)$ for the barrier on the right.}
    \label{fig:all_tests}
\end{figure*}

\renewcommand{\baselinestretch}{0.99}
\bibliographystyle{IEEEtran}
\bibliography{mybib}

\end{document}